\newif\ifArxiv
\Arxivtrue

\ifArxiv
    \wlog{Using ArXiv version}
\else
    \wlog{Using venue version}
\fi

\ifArxiv
    \documentclass[]{article}

    \usepackage{fullpage}
    \usepackage{natbib}
    \usepackage{hyperref}
    \usepackage[x11names]{xcolor}
    \usepackage[vlined,algoruled,noend,linesnumbered,algo2e]{algorithm2e}

    \usepackage{amsthm}
    \newtheorem{theorem}{Theorem}[section]
    \newtheorem{lemma}[theorem]{Lemma}

    \theoremstyle{definition}

    \newcommand{\acks}[1]{\section*{Acknowledgments}#1}
\else
    \PassOptionsToPackage{lined,noend,linesnumbered}{algorithm2e}
    \documentclass[final,12pt]{alt2025} %

\fi

\ifArxiv
    \title{Boosting, Voting Classifiers and\texorpdfstring{\\}{} Randomized Sample Compression Schemes}
    \author{%
        Arthur {da Cunha}
        \qquad
        Kasper Green Larsen\\
        Aarhus University\\
        \texttt{\{dac, larsen\}@cs.au.dk}
        \and
        Martin Ritzert\\
        Georg-August Universit\"at G\"ottingen\\
        \texttt{ritzert@informatik.uni-goettingen.de}
    }
    \date{}
\else
    \title[Randomized Sample Compression Schemes]{Boosting, Voting Classifiers and\texorpdfstring{\\}{} Randomized Sample Compression Schemes}

    \altauthor{\Name{Arthur {da Cunha}} \Email{dac@cs.au.dk}\and
     \Name{Kasper Green Larsen} \Email{larsen@cs.au.dk}\\
     \addr Aarhus University
     \AND
     \Name{Martin Ritzert} \Email{ritzert@informatik.uni-goettingen.de}\\
     \addr Georg-August Universit\"at G\"ottingen%
    }

\fi
\usepackage{times}

\SetCommentSty{mycommentfont}
\usepackage{etoolbox}
\makeatletter
\patchcmd{\@algocf@start}%
{-1.5em}%
{-2pt}%
{}{}%
\makeatother
\let\oldnl\nl%
\newcommand{\nonl}{\renewcommand{\nl}{\let\nl\oldnl}}%

\usepackage{amsmath}
\usepackage{amssymb}
\usepackage{mathtools}

\providecommand\given{}
\newcommand\GivenSymbol[1][]{%
    \mathchoice{\:}{\:}{\,}{\,}#1\vert%
    \allowbreak%
    \mathchoice{\:}{\:}{\,}{\,}%
    \mathopen{}%
}
\DeclarePairedDelimiterX\Set[1]\{\}{%
    \renewcommand\given{\GivenSymbol[\delimsize]}%
    #1%
}
\DeclarePairedDelimiterXPP\ParensWithGiven[1]{}{(}{)}{}{%
    \renewcommand\given{\GivenSymbol[\delimsize]}%
    #1%
}
\DeclarePairedDelimiterXPP\BracksWithGiven[1]{}{[}{]}{}{%
    \renewcommand\given{\GivenSymbol[\delimsize]}%
    #1%
}
\NewDocumentCommand\Prob{ e{_^} }{
    \operatorname*{Pr}
    \IfNoValueF{#1}{\sb{#1}}
    \IfNoValueF{#2}{\errmessage{Not supposed to have superscript}\sp{#2}}
    \BracksWithGiven
}
\NewDocumentCommand\Ev{ e{_^} }{
    \operatorname{\mathbb{E}}
    \IfNoValueF{#1}{\sb{#1}}
    \IfNoValueF{#2}{\errmessage{Not supposed to have superscript}\sp{#2}}
    \BracksWithGiven
}
\NewDocumentCommand\Var{ e{_^} }{
    \operatorname{Var}
    \IfNoValueF{#1}{\sb{#1}}
    \IfNoValueF{#2}{\errmessage{Not supposed to have superscript}\sp{#2}}
    \ParensWithGiven
}
\NewDocumentCommand\Cov{ e{_^} }{
    \operatorname{Cov}
    \IfNoValueF{#1}{\sb{#1}}
    \IfNoValueF{#2}{\errmessage{Not supposed to have superscript}\sp{#2}}
    \ParensWithGiven
}

\newcommand{\R}{\mathbb{R}}
\newcommand{\N}{\mathbb{N}}

\DeclareMathOperator{\Normal}{\mathcal{N}}

\DeclarePairedDelimiter\abs\lvert\rvert
\DeclarePairedDelimiterXPP\normone[1]{}\lVert\rVert{_1}{#1}
\DeclarePairedDelimiterXPP\normtwo[1]{}\lVert\rVert{_2}{#1}
\DeclarePairedDelimiterXPP\norminf[1]{}\lVert\rVert{_\infty}{#1}
\DeclarePairedDelimiterXPP\normmax[1]{}\lVert\rVert{_\mathrm{max}}{#1}
\DeclarePairedDelimiterXPP\normspec[1]{}\lVert\rVert{_\mathrm{spectral}}{#1}
\let\normmax\norminf

\newcommand*{\methodname}[1]{\normalfont{\textsc{#1}}}

\newcommand{\cX}{\mathcal{X}}
\newcommand{\cY}{\mathcal{Y}}

\newcommand{\bS}{\mathbf{S}}
\newcommand{\bx}{\mathbf{x}}
\newcommand{\cA}{\mathcal{A}}
\newcommand{\brf}{\mathbf{f}}
\newcommand{\br}{\mathbf{r}}
\newcommand{\bg}{\mathbf{g}}
\newcommand{\eps}{\varepsilon}
\newcommand{\by}{\mathbf{y}}
\newcommand{\bX}{\mathbf{X}}

\newcommand{\cH}{\mathcal{H}}
\newcommand{\cW}{\mathcal{W}}
\newcommand{\bZ}{\mathbf{Z}}
\newcommand{\bD}{\mathbf{D}}
\newcommand{\cD}{\mathcal{D}}
\usepackage{bm}
\newcommand{\randKappa}{\bm{\kappa}}
\newcommand{\bh}{\mathbf{h}}

\newcommand{\Risk}{R}
\DeclareMathOperator{\sign}{sign}
\DeclareMathOperator{\margin}{margin}

\let\cite\citep

\begin{document}
    \maketitle

    \begin{abstract}%
        In \emph{boosting}, we aim to leverage multiple \emph{weak learners} to produce a \emph{strong learner}.
At the center of this paradigm lies the concept of building the strong learner as a \emph{voting classifier}, which outputs a weighted majority vote of the weak learners.
While many successful boosting algorithms, such as the iconic AdaBoost, produce voting classifiers, their theoretical performance has long remained sub-optimal: The best known bounds on the number of training examples necessary for a voting classifier to obtain a given accuracy has so far always contained at least two logarithmic factors above what is known to be achievable by general \emph{weak-to-strong} learners.
In this work, we break this barrier by proposing a randomized boosting algorithm that outputs voting classifiers whose generalization error contains a single logarithmic dependency on the sample size.
We obtain this result by building a general framework that extends sample compression methods to support randomized learning algorithms based on sub-sampling.

    \end{abstract}

    \ifArxiv\else
    \fi
    \section{Introduction}
Boosting is a powerful machine learning primitive that allows improving the performance of a base learning algorithm $\cA$ by training a committee/ensemble of classifiers.
The classic AdaBoost~\citep{adaboost} algorithm for binary classification is perhaps the most well-known boosting algorithm.
Given an input domain $\cX$ and a set $S = \{(x_1, y_1), \ldots, (x_n,y_n)\}$ of $n$ labeled samples from $\cX \times \{-1, 1\}$, the main idea of AdaBoost is to iteratively invoke $\cA$ on reweighed versions of $S$.
Each invocation returns a hypothesis $h_t\colon \cX \to \{-1,1\}$ to be combined into a final \emph{voting classifier} $f$ as $f(x) = \sign(\sum_{t=1}^T \alpha_t h_t(x))$ for constants $\alpha_t>0$.
The weights used at iteration $t$ are such that samples $(x_i,y_i)$ that are misclassified by many previous hypotheses $h_j$ with $j < t$ receive a large weight, and correctly classified samples receive smaller weights.
This intuitively guides the attention of $\cA$ towards samples with which that previous hypotheses struggle.
More modern variants of boosting include the highly practical XGBoost~\cite{xgboost} and LightGBM~\cite{lightGBM} implementations of Gradient Boosting~\cite{gradboost}.
See the survey by~\citet{survey} for more on boosting and its applications.

\paragraph{Weak-to-Strong Learning.}
Historically, boosting was invented to address a theoretical question of~\citet{kearns1988learning, kearns1994cryptographic} on weak-to-strong learning.
A $\gamma$-weak learner $\cW$ is a learning algorithm which, when queried with a training set $S$ and a distribution $\cD$ over $S$, returns a hypothesis $h$ with $\Risk_\cD(h) \leq 1/2-\gamma$.
Here $\Risk_\cD(h) = \Pr_{(\bx,\by) \sim \cD}[h(\bx)\neq \by]$.
An $(\eps,\delta)$-strong learner on the other hand, is a learning algorithm such that for any distribution $\cD$ over $\cX \times \{-1,1\}$, when given $m(\eps,\delta)$ i.i.d.\ samples from $\cD$, returns with probability at least $1 - \delta$ a hypothesis $f\colon \cX \to \{-1,1\}$ with $\Risk_\cD(f) \leq \eps$.
A strong learner may, thus, achieve arbitrarily high accuracy when given enough samples.

With these definitions, Kearns and Valiant asked whether it is always possible to obtain a strong learner from a weak learner.
This was answered affirmatively~\cite{schapire1990strength}, and AdaBoost is the prototypical such weak-to-strong learner.
A natural question is: Given $n$ samples, what is the smallest $\Risk_\cD(f)$ achievable for a weak-to-strong learner when given access to a $\gamma$-weak learner $\cW$?
Letting $\cH$ denote a hypothesis set such that $\cW$ always outputs hypotheses from $\cH$, if $\cH$ has VC-dimension $d$, \citet{understandingMachineLearning} showed that with probability greater than $1-\delta$, AdaBoost outputs a voting classifier $f$ with
\begin{equation}
\label{eq:ada}
\Risk_{\cD}(f) = O\left(\frac{d \ln(n/d) \ln n}{\gamma^2 n} + \frac{\ln(1/\delta)}{n}\right).
\end{equation}
This bound remains the best known for any weak-to-strong learner that outputs a voting classifier: One which makes predictions by taking a weighted majority vote among a set of base classifiers.

On the lower bound side, \citet{optimalWeakToStrong} showed that for any weak-to-strong learner, with constant probability over a set of $n$ training samples, the produced hypothesis $f$ satisfies
\[
    \Risk_{\cD}(f) = \Omega\left(\frac{d}{\gamma^2 n}\right).
\]
Note that this holds for all weak-to-strong learners, not just those that output a voting classifier.
Furthermore, they complemented the lower bound by a boosting algorithm achieving an optimal
\begin{equation}
\label{eq:opt}
\Risk_{\cD}(f) = O\left( \frac{d}{\gamma^2 n} + \frac{\ln(1/\delta)}{n}\right).
\end{equation}
Thus, at a high level, the sample complexity of weak-to-strong learning is fully understood.
However, the algorithm by \citeauthor{optimalWeakToStrong} is somewhat contrived as the produced hypothesis is a majority-of-majorities and \emph{not} a voting classifier.
Concretely, using recent results to simplify their algorithm~\cite{baggingIsOptimal}, \citeauthor{optimalWeakToStrong} combine classic Bagging by~\citet{bagging} with a variant of AdaBoost known as AdaBoost$^*_{\nu}$ \cite{ratsch2005efficient}.
They thus create multiple sub-samples of the training data, train a voting classifier on each, and combine them by taking a majority of their predictions.

\paragraph{Contribution I: A New Voting Classifier.}
In light of the above, it remains a natural and basic theoretical question whether the optimal weak-to-strong learning sample complexity in Eq.~\eqref{eq:opt} can be achieved by a simple voting classifier. 

Our first main contribution is a new boosting algorithm, shown as Algorithm~\ref{alg:adaboost}, that produces a voting classifier with an improved generalization error in terms of the sample size $n$.
In the algorithm description, $a>0$ is a sufficiently large constant. We prove the following sample complexity bound for Algorithm~\ref{alg:adaboost}:
\begin{theorem}
    \label{thm:boostintro}
    There exists universal constant $C > 0$ for which the following holds.
    Let $\cD$ be an unknown distribution over $\cX \times \{-1,1\}$ and let $\bS \sim \cD^n$.
    Then for every $\delta>0$, it holds with probability at least $1-\delta$ over $\bS$ and the randomness of Algorithm~\ref{alg:adaboost} with $\bS$, $\delta$, a $\gamma$-weak learner $\cW$ and $N=n$ as input, that the voting classifier $\bg = \sign(\brf)$ produced satisfies
    \[
        \Risk_\cD(\bg)
        \leq C \cdot \min\left\{
            \frac{(d + \ln(1/\gamma))\ln(n/\delta)}{\gamma^4 n}
            ,\,
            \frac{d \ln(n/d) \ln n}{\gamma^2 n}  + \frac{\ln(1/\delta)}{n}
        \right\}.
    \]
\end{theorem}
While it can reduce to the previous best bounds in some regimes, it is the first voting classifier that can achieve a sample complexity with a single logarithmic dependency on $n$.

\begin{algorithm2e}[ht]
    \SetKw{Break}{break}
    \SetKw{KwReturn}{return}
    \DontPrintSemicolon
    \KwIn{Training set $S = \{(x_1,y_1),\dots,(x_n,y_n)\}$,
        $\gamma$-weak learner $\cW$, failure probability $\delta$,\\
        \quad upper bound $N \geq n$.
    }
    \KwResult{A voting  classifier $f$.}
    
    $\cD_1 \gets \left( \frac{1}{n}, \dots, \frac{1}{n} \right)$
    
    $\alpha \gets \frac{1}{2} \ln\frac{1/2+\gamma/2}{1/2-\gamma/2}$ \tcp*{guaranteed instead of empirical error}
    
    $m \gets a \cdot \gamma^{-2}(d+\ln(1/\gamma))$ \tcp*{subsample size}
    
    $K \gets 32 \cdot (\gamma^{-2} \ln(N/\delta)+1)$ \tcp*{fixed size of final ensemble}
    
    \For(){$k= 1,\dots,K$}{ \label{line:forloop}
    
        Draw $m$ samples $\bS_k \sim \bD_k^{m}$ \label{line:sample}
        
        Invoke $\cW$ on $\bS_k$ with the uniform distribution to obtain $\bh_k$ \label{line:W}
        
        \For(\tcp*[f]{standard AdaBoost weight update}){$i = 1, \ldots, n$}{
            $\bD_{k+1}(i) \gets \bD_k(i) \exp(-\alpha y_i \bh_k(x_i))$
        }
        
        $\bZ_k \gets \sum_{i=1}^n \bD_k(i)\exp(-\alpha y_i \bh_k(x_i))$
        
        $\bD_{k+1} \gets \bD_{k+1}/\bZ_k$
    }
    \KwReturn{$\brf(x) = \frac{1}{K}\sum_{k=1}^K \bh_k(x)$}
    \tcp*{majority vote}
    \caption{Sampled Boosting}\label{alg:adaboost}
\end{algorithm2e}

At a high level, our new algorithm creates numerous small sub-samples of the training data and combines classifiers trained on each of them.
Proving that this is beneficial requires highly novel analysis techniques.
Our second main contribution is thus a new general framework for analyzing randomized learning algorithms that use sub-sampling during training.
This method builds on the sample compression framework of~\citet{compression} and we hope it may prove useful in the future development and analysis of efficient learning algorithms.
We introduce this new framework in the following subsection and then discuss the connection between Algorithm~\ref{alg:adaboost} and the framework.

\subsection{Sample Compression Schemes}
Learning and compression have been known to be tightly connected for decades.
One of the earliest and clearest connections between the two originates in the work of~\citet{compression}.
In essence, they argue that if the hypothesis produced by a learning algorithm can be \emph{compressed} to be fully described as a function of a few training samples, then it generalizes well.
We describe this connection further in the following.

Let $\cX$ be an input domain and $\cY$ an output domain.
A compression scheme $(\kappa, \rho)$ consists of an encoding map $\kappa$ that maps any sequence $S \in (\cX \times \cY)^*$ to a subsequence $\kappa(S)$ of $S$, and a reconstruction function $\rho\colon (\cX \times \cY)^* \to \cY^\cX$ mapping any $S \in (\cX \times \cY)^*$ to a function $\rho(S)\colon \cX \to \cY$.
The compression scheme must satisfy for any $S$ that $\rho(\kappa(S))(x) = y$ for all $(x,y) \in S$.
The size of the compression scheme is the supremum over $S$ of $\abs{\kappa(S)}$, for given a given size of $S$.
Notably, some notions of compression schemes forgo this dependency on the sample size, e.g., in \citet{MoranY16}.

Consider now a learning algorithm $\cA$ and assume there is a corresponding compression scheme $(\kappa, \rho)$ of size $s$, such that when $\cA$ produces a hypothesis $h_S\colon \cX \to \cY$ from a training set $S$, then the corresponding compression scheme satisfies $\rho(\kappa(S)) = h_S$.
In this case, we can prove a bound on the generalization of $h_\bS$ for a training set $\bS \sim \cD^n$.
In a nutshell, we observe that there are only $M=\sum_{i \leq s} \binom{n}{i}$ possible choices for $\kappa(\bS)$.
Since $\rho(\bS')$ for a fixed subset $\bS' \subseteq \bS$ is determined from the samples in $\bS'$ alone, and the remaining $n -\abs{\bS'}$ samples are i.i.d.\ from $\cD$, a union bound over the $M$ choices for $\bS'$ shows that with probability at least $1-\delta$, there is no $\bS'$ with $\rho(\bS')(\bx) = \by$ for all $(\bx,\by) \in \bS$ and yet $\Risk_\cD(\rho(\bS'))$ is larger than $O(\ln(M)/n + \ln(1/\delta)/n) = O((s \ln(n/s) + \ln(1/\delta))/n)$.
Thus, in particular, $\Risk_\cD(h_\bS) = \Risk_\cD(\rho(\kappa(\bS))) = O((s \ln(n/s) + \ln(1/\delta))/n)$.

Interestingly, the factor $\ln(n/s)$ in the generalization bound can be removed if the compression scheme satisfies an additional property of \emph{stability} introduced by~\citet{stable}.
A compression scheme is \emph{stable} if for any training set $S$ and subset $S'$ with $\kappa(S) \subseteq S' \subseteq S$, it holds that $\rho(\kappa(S))=\rho(\kappa(S'))$.
In words, if we remove training samples not part of the compression $\kappa(S)$ from $S$, then the resulting training set $S'$ is still compressed to the same.
\citet{stable} proved the first tight generalization bounds for Support Vector Machines by constructing a suitable stable sample compression scheme.

\paragraph{Contribution II: Randomized Compression Schemes.}
Our work introduces the notion of a randomized compression scheme and use it to prove generalization of Algorithm~\ref{alg:adaboost}.
Such a randomized compression scheme $(\cD_\kappa, \rho)$ consists of a distribution $\cD_\kappa$ over encoding maps, and a reconstruction function $\rho$ that is not randomized, but simply defined as for regular compression schemes.

As a further extension to the standard compression framework, we give $\kappa$ an upper bound $n$ of the cardinality of the training sample considered.
Furthermore, we allow a bit more freedom in the encoding by not requiring $\kappa(S)$ to be a subsequence of $S$.
More precisely,
\begin{itemize}
    \item The distribution $\cD_\kappa$ is over (deterministic) encoding functions $\kappa$ that map any sequence $S \in (\cX \times \cY)^*$ and integer $n \geq |S|$, to a sequence $\kappa(S,n)$ such that every element of $\kappa(S,n)$ appears in $S$. 
\end{itemize}
We dedicate the symbol ``$\sqsubseteq$'' to represent that every element of a sequence appears in another sequence.
Formally, given sequences $S = (s_1, \ldots, s_m)$ and $T = (t_1, \ldots, t_n)$, we write $S \sqsubseteq T$ if and only if $\Set{s_i \given i \in [m]} \subseteq \Set{t_j \given j \in [n]}$.

Note that the definition above allows the samples in $\kappa(S,n)$ to appear in a different order than in $S$ and to appear a different number of times.

A randomized compression scheme has failure probability at most $\delta$ if for all $S \in (\cX \times \cY)^*$ and $n \geq \abs{S}$ it holds that
\[
    \Pr_{\randKappa \sim \cD_\kappa}[\exists (x,y) \in S : \rho(\randKappa(S,n))(x) \neq y] \leq \delta.
\]

A randomized compression scheme $(\cD_\kappa, \rho)$ is \emph{stable} if and only if given i.i.d.\ $\randKappa, \randKappa' \sim \cD_\kappa$, for any $S \in (\cX \times \cY)^*$ and $n \in \N$ with $n \geq \abs{S}$, and any subsequence $S'$ of $S$ in the support of $\randKappa(S,n)$, the distribution of $\randKappa'(S',n)$ is the same as the distribution of $\randKappa(S,n)$ conditioned on $\randKappa(S,n) \sqsubseteq S'$.
That is, for all $T \in (\cX \times \cY)^*$, we have that
\[
    \Pr[\randKappa'(S', n) = T] = \Pr[\randKappa(S, n) = T \mathrel{\vert} \randKappa(S, n) \sqsubseteq S'].
\]

Given $n \in \N$, the \emph{size} $s_n$ of a randomized compression scheme is the supremum over $(S,j)$ in $\cup_{i = 1}^n ((\cX \times \cY)^i \times \{i, \dots, n\})$, and $k$ in the support of $\cD_\kappa$, of the number of distinct $(x,y)$ in $\kappa(S,j)$.

Our main technical result for proving generalization via randomized compression is the following theorem:

\begin{theorem}
\label{thm:genstable}
    There exists universal constant $C > 0$ for which the following holds.
    Let $\cD$ be an unknown distribution over $\cX \times \cY$ and let $\bS \sim \cD^n$.
    Let $(\cD_\kappa,\rho)$ be a stable randomized compression scheme with failure probability at most $\delta$ and size $s = s_n$.
    Then for every $\beta>2 \delta$, it holds with probability at least $1-\beta$ over $\bS$ and $\randKappa \sim \cD_\kappa$ that
    \[
        \Risk_\cD(\rho(\randKappa(\bS,n))) \leq  C \cdot \frac{s + \ln(1/\beta)}{n},
    \]
    where $\Risk_\cD(h) = \Pr_{(\bx,\by)\sim \cD}[h(\bx) \neq \by]$.
\end{theorem}
Similarly to the stable compression schemes of~\citet{stable}, the generalization bound in Theorem~\ref{thm:genstable} depends linearly on $s$ and not as $s \ln(n/s)$ like the bounds of~\citet{compression} without stability.

In light of Theorem~\ref{thm:genstable}, we prove generalization of our new boosting algorithm, Algorithm~\ref{alg:adaboost}, by showing that there is a corresponding randomized compression scheme of size $s = s_n = O((d+\ln(1/\gamma))\ln(n/\delta)/\gamma^4)$ and invoking Theorem~\ref{thm:genstable}.

\subsection{Main Ideas in Algorithm~\ref{alg:adaboost}}
Having presented our randomized compression framework, let us now discuss the main ideas and obstacles overcome by Algorithm~\ref{alg:adaboost} and how they relate to randomized compression.
We also argue why the classic compression frameworks are insufficient for our purpose, thus further motivating our randomized framework.

In striving to improve the sample complexity of voting classifiers, a natural approach would be to apply the classic stable compression framework of~\citet{stable}, as it is known to improve sample complexity by a logarithmic factor.
However, combining classic sample compression with boosting appears tricky.
To see this, notice that boosting algorithms invoke a weak learner $\cW$ with a distribution $D$ over the full training set $S$.
The weak learner then returns a hypothesis $h_D$, depending on $D$, that is used in a final classifier $f$.
For the purpose of invoking a compression framework to argue generalization of $f$, we would like to argue that a small subset $\kappa(S) \subseteq S$ may be used to reconstruct $f$.
However, we have no control over the weak learner $\cW$ and it is completely unclear that we would be able to recover each $h_D$ used in $f$ without including all of $S$ in $\kappa(S)$.

For the reader familiar with AdaBoost, Algorithm~\ref{alg:adaboost} is seen to resemble it quite closely.
However, for standard AdaBoost, the weak learner $\cW$ would be invoked directly on the distributions $\bD_k$ in Algorithm~\ref{alg:adaboost}.
In order to give an efficient compression, we instead draw samples $\bS_k \sim \bD_k^m$ and invoke $\cW$ on just the samples.
This way, we can intuitively reconstruct the hypotheses $\bh_k$ from just the samples $\bS_1,\dots,\bS_K$ and this is precisely what we do in our proof of Theorem~\ref{thm:boostintro}, i.e.\ we let our encoding be the samples in $\bS_1,\dots,\bS_K$.

Still, we need the final classifier produced by Algorithm~\ref{alg:adaboost} to be correct on the training data (the compression scheme must have small failure probability).
This puts a constraint on the number of samples $m$ and iterations $K$.
Here we use an observation from previous work~\cite{parallel} on parallel boosting, showing that the set $\bS_k$ forms a $(\gamma/2)$-approximation for the distribution $\bD_k$ with good probability (see the correctness proof for details).
At a high level, this implies that the hypothesis $\bh_k$ returned by the weak learner has error at most $1/2 - \gamma/2$  under $\bD_k$.
A mostly standard analysis of AdaBoost then shows that after $K$ iterations, the resulting voting classifier $\brf$ is correct on all the training data (and thus the compression scheme has small failure probability).

A natural question is whether we really need the randomness from our new framework, or the classic stable compression framework by~\citet{stable} would suffice.
To use their framework, we would need to \emph{deterministically} pick the sets $\bS_k$.
While it is known that a random $\bS_k \sim \bD^m$ forms a $\gamma/2$-approximation with constant probability when $m = \Omega(d/\gamma^2)$, it is not clear how to compute such a set deterministically in time less than the number of distinct hypotheses from which the weak learner might choose, which may be as large as $\binom{n}{d}$ when constrained to $S$.

In light of the above, our new randomized compression framework provides means to analyzing learning algorithms that use random sampling to quickly find sub-samples $\bS' \subset S$ with desirable properties that are hard to guarantee deterministically.

Finally, we overview the stability of Algorithm~\ref{alg:adaboost} (formal details appear later).
That is, we need to argue that for any subsequence $S' \subseteq S$ of the training data, if we condition on $\bS_1,\dots,\bS_K \sqsubseteq S'$, then the distribution of $\bS_1,\dots,\bS_K$ is the same as the distribution of $\bS'_1,\dots,\bS'_K$ resulting from instead running Algorithm~\ref{alg:adaboost} on the input $S'$.
We argue this by induction roughly as follows:
Assume we have already shown it for the prefix $\bS_1,\dots,\bS_k$ and $\bS'_1,\dots,\bS'_k$.
Then the distribution of the hypotheses $\bh_1,\dots,\bh_k$ and $\bh'_1,\dots,\bh'_k$ in the two executions would be identical.
Now for any $h_1,\dots,h_k$ in the support of this distribution, the weights in $\bD_{k+1}$ and $\bD'_{k+1}$ computed by Algorithm~\ref{alg:adaboost} are completely determined as $\bD_{k+1}(j) = \exp(-y_j \sum_{\ell=1}^k \alpha h_\ell(x_j))/Z$ and $\bD'_{k+1}(j) = \exp(-y_j \sum_{\ell=1}^k \alpha h_\ell(x_j))/Z'$ where $Z$ and $Z'$ are normalization factors making $\bD_{k+1}$ and $\bD'_{k+1}$ probability distributions.
The crucial point is that the  ``weight" of each point $x_j \in S'$ is the same in $\bD_{k+1}$ and $\bD'_{k+1}$ up to the normalization terms $Z$ and $Z'$.
When we further condition on $\bS_{k+1} \subseteq S'$, this effectively rescales $\bD_{k+1}$ by setting all weights outside $S'$ to $0$ and changing the normalization factor to $Z'$, making the distribution the same as for $\bS'_{k+1}$. 

\subsection{Other Related Work}
\label{sec:other}
Let us finally describe other relevant previous works, in particular results showing barriers for further improving the sample complexity of voting classifiers.

First, one natural approach to training a voting classifier $f(x) = \sign(\sum_t \alpha_t h_t(x))$ with a sample complexity matching the best previously known for voting classifiers (Eq.~\eqref{eq:ada}) is to ensure that $f$ has all \emph{margins} on the training data $\Omega(\gamma)$.
The margin of $f$ on a sample $(x,y)$ is defined as
\[
    \margin_f(x,y) \coloneqq y \cdot \frac{\sum_t \alpha_t h_t(x)}{\sum_t \abs{\alpha_t}}.
\]
Margins were originally introduced to explain the excellent practical performance of AdaBoost and its variants~\cite{bartlett}.
Several \emph{uniform convergence} based generalization bounds have been shown for large margin voting classifiers~\cite{bartlett,breiman1999prediction}, with the state-of-the-art being the $k$th margin bound by~\citet{gao}.
Simplified to all margins being at least $\gamma$, they showed that with probability at least $1-\delta$ over a set of $n$ training samples from a distribution $\cD$, it simultaneously holds that \emph{all} voting classifiers $f$ with all margins on the training data at least $\gamma$ satisfy that
\begin{equation}
    \Risk_{\cD}(f) = O\left(\frac{d \ln(n/d) \ln n}{\gamma^2 n}  + \frac{\ln(1/\delta)}{n}\right).
    \label{eq:marginbased:ub}
\end{equation}
Here $d$ denotes the VC-dimension of the hypothesis set $\cH$ to which all $h_t$ in the voting classifiers $f$ belong.
AdaBoost$^*_{\nu}$ \cite{ratsch2005efficient} is a boosting algorithm that outputs a voting classifier guaranteed to have all margins $\Omega(\gamma)$.
Using Eq.~\eqref{eq:marginbased:ub} yields the previously best sample complexity of voting classifiers stated in Eq.~\eqref{eq:ada} for the AdaBoost$^*_{\nu}$ algorithm.\footnotemark
\footnotetext{In fact, to prove Theorem~\ref{thm:boostintro} we too argue that Algorithm~\ref{alg:adaboost} has large margins, leading to the bound being expressed as a minimum by leveraging Eq.~\eqref{eq:ada}.}

It follows that if the uniform convergence bound for large margin voting classifiers could be strengthened to $O(d/(\gamma^2 n) + \ln(1/\delta)/n)$, then AdaBoost$^*_{\nu}$ would be an optimal weak-to-strong learner.
Unfortunately, lower bounds against uniform convergence~\cite{boostingLowerBound,marginsGradient} show example distributions and hypothesis sets such that with constant probability over $n$ samples, \emph{there exists} a voting classifier $f$ with all margins at least $\gamma$ and yet
\begin{equation}
    \label{eq:badvoter}
    \Risk_{\cD}(f)= \Omega\left(\frac{d \ln(\gamma^2 n/d) }{\gamma^2 n}\right). 
\end{equation}
Abandoning the hope of proving that a voting classifier is optimal via uniform convergence, a natural goal would be to show that a concrete boosting algorithm, like AdaBoost or AdaBoost$^*_{\nu}$ is optimal, i.e.\ to exploit concrete properties of the boosting algorithm to argue for better generalization than that in Eq.~\eqref{eq:badvoter}.
However, recent work~\cite{adaboostNotOptimal} shows that all previous boosting algorithms that produce voting classifiers, satisfy that with constant probability over $n$ samples, the produced voting classifier has a sample complexity of at least that in Eq.~\eqref{eq:badvoter}.
At a high level, the work of~\citet{adaboostNotOptimal} shows that any boosting algorithm that always invokes the weak learner $\cW$ with a distribution $\cD$ having support on the full training data set has a generalization error of at least Eq.~\eqref{eq:badvoter}.
The only known boosting algorithms avoiding this pitfall is the optimal, but non-voting classifier, by~\citet{optimalWeakToStrong}, and our new Algorithm~\ref{alg:adaboost}.

In summary, several barriers need to be overcome to avoid at least one logarithmic factor overhead in the sample complexity as a function of $n$.

\subsection{Preliminaries}
Throughout the paper, we assume for simplicity that the training sets contain no duplicates.
One can see that this assumption does not reduce the generality of our arguments by, e.g., letting $\cX' = \cX \times [0,1]$ and changing the input distribution $\cD$ to $\cD'$ over $\cX' \times \cY$, where $\cD'$ generates a pair $(\bx',\by)$ by letting $\bx' = (\bx, \br)$ for $(\bx,\by) \sim \cD$ and $\br \sim \operatorname{Uniform}([0,1])$.
The weak learner then simply ignores $\br$.
Finally, as the reader may have noticed, we reserve boldface letters for random variables (e.g., $x \in \R$ vs.~$\bx \sim \Normal(0, 1)$).

    \section{Generalization via Randomized Compression}
In this section, we prove Theorem~\ref{thm:genstable} which establishes generalization via randomized compression schemes.
So, let $\bS \sim \cD^n$ be a training set of size $n$ and let $s = s_n$. 
\begin{proof}[Proof of Theorem~\ref{thm:genstable}]
    Partition $\bS$ into $2s$ buckets of $n/2s$ samples each and denote these buckets by $\bS_1,\dots,\bS_{2s}$.
For every subset $I \in \binom{[2s]}{s}$ of $s$ indices of buckets, let $\bS_I$ denote the concatenation of the samples in buckets $\bS_i$ with $i \in I$. Here the notation $\binom{[2s]}{s}$ refers to all subsets of $[2s]$ of cardinality $s$. Finally, define $\bar{\bS}_I$ as the concatenation of the buckets $\bS_i$ with $i \notin I$.

Now consider a random $\randKappa \sim \cD_\kappa$.
For each $I \in \binom{[2s]}{s}$,
let $E_{I,\randKappa}$ denote the event that $\randKappa(\bS,n) \sqsubseteq \bS_I$,
which we denote simply as $E_I$ when $\randKappa$ is clear from the context.
Notice that $\Pr[\cup_I E_I] = 1$ since the size of the compression scheme is $s$.

Next, for each $I$ and parameter $\alpha > 0$ define $p_{I,\alpha}$ to be the probability
\begin{equation*}
    \Pr_{\substack{\randKappa \sim \cD_\kappa,\\\bS_I,\bar{\bS}_I \sim \cD^{n/2}}}\left[\begin{gathered}
        \forall (x,y) \in \bar{\bS}_I, \rho(\randKappa(\bS_I,n))(x) = y\;
        \wedge\; \Risk_\cD(\rho(\randKappa(\bS_I,n))) \geq \alpha
    \end{gathered}\right].
\end{equation*}
To bound $p_{I,\alpha}$, fix any $S_I$ and $\kappa$ in the supports of $\bS_I$ and $\randKappa$. If $\Risk_\cD(\rho(\kappa(S_I,n))) < \alpha$, then $S_I$ and $\kappa$ contribute $0$ to $p_{I,\alpha}$. Otherwise, since $\bar{\bS}_I$ is independent of $\bS_I$, we have that $\Pr_{\bar{\bS}_I \sim \cD^{n/2}}[\forall (x,y) \in \bar{\bS}_I, \rho(\kappa(\bS_I,n))(x)=y] \leq (1-\alpha)^{n/2} \leq \exp(-\alpha n/2)$. Thus $p_{I,\alpha} \leq \exp(-\alpha n/2)$.

Moreover, it holds that
\begin{align*}
    \Pr_{\substack{\randKappa \sim \cD_\kappa,\\ \bS \sim \cD^n}}[\Risk_\cD(\rho(\randKappa(\bS,n))) \geq \alpha]
    &\leq \Pr_{\randKappa, \bS}[\exists (x,y) \in \bS : \rho(\randKappa(\bS,n))(x) \neq y]\\
    &\quad+ \Pr_{\randKappa, \bS}\left[\begin{gathered}
        \forall(x, y) \in \bS, \rho(\randKappa(\bS,n))(x) = y\;
        \wedge\; \Risk_\cD(\rho(\randKappa(\bS,n))) \geq \alpha
    \end{gathered}\right].
\end{align*}
By definition, we have $\Pr[\exists (x,y) \in \bS : \rho(\randKappa(\bS,n))(x) \neq y] < \delta$.
Also, since $\cup_I E_I$ always occur,
\begin{align*}
    &\Pr_{\substack{\randKappa \sim \cD_\kappa,\\ \bS \sim \cD^n}}\left[\begin{gathered}
        \forall(x, y) \in \bS, \rho(\randKappa(\bS,n))(x) = y\;
        \wedge\; \Risk_\cD(\rho(\randKappa(\bS,n))) \geq \alpha
    \end{gathered}\right]\\
    &\qquad= \Pr_{\randKappa, \bS}\left[\begin{gathered}
        \forall(x, y) \in \bS, \rho(\randKappa(\bS,n))(x) = y\;
        \wedge\; \Risk_\cD(\rho(\randKappa(\bS,n))) \geq \alpha \;\wedge\; \cup_I E_I
    \end{gathered}\right]\\
    &\qquad\leq \sum_I \Pr_{\randKappa, \bS}\left[\begin{gathered}
        \forall(x, y) \in \bS, \rho(\randKappa(\bS,n))(x) = y\;
        \wedge\; \Risk_\cD(\rho(\randKappa(\bS,n))) \geq \alpha \;\wedge\; E_I
    \end{gathered}\right]\\
    &\qquad= \sum_I \Pr_{\randKappa, \bS}\left[\begin{gathered}
        \forall(x, y) \in \bS, \rho(\randKappa(\bS,n))(x) = y\;
        \wedge\; \Risk_\cD(\rho(\randKappa(\bS,n))) \geq \alpha \mid E_I
    \end{gathered}\right] \cdot \Pr_{\randKappa, \bS}[E_I].
\end{align*}
Now observe that since $(\cD_\kappa, \rho)$ is a stable randomized compression scheme, the distribution of $\rho(\randKappa(\bS,n))$ conditioned on $E_I$ is the same as $\rho(\randKappa'(\bS_I,n))$ for a fresh $\randKappa' \sim \cD_\kappa$.
Thus,
\begin{align*}
    &\sum_I \Pr_{\substack{\randKappa \sim \cD_\kappa,\\ \bS \sim \cD^n}}\left[\begin{gathered}
        \forall(x, y) \in \bS, \rho(\randKappa(\bS,n))(x) = y\;
        \wedge\; \Risk_\cD(\rho(\randKappa(\bS,n))) \geq \alpha \mid E_I
    \end{gathered}\right]\cdot \Pr_{\randKappa, \bS}[E_I]\\
    &\quad= \sum_I \Pr_{\substack{\randKappa \sim \cD_\kappa,\\\randKappa' \sim \cD_\kappa,\\ \bS \sim \cD^n}}\left[\begin{gathered}
        \forall(x, y) \in \bS, \rho(\randKappa'(\bS_I,n))(x) = y\;
        \wedge\; \Risk_\cD(\rho(\randKappa'(\bS_I,n))) \geq \alpha
    \end{gathered} \mid E_{I,\randKappa} \right]\cdot \Pr_{\randKappa, \bS}[E_{I,\randKappa}]\\
    &\quad\leq \sum_I \Pr_{\substack{\randKappa' \sim \cD_\kappa,\\ \bS \sim \cD^n}}\left[\begin{gathered}
        \forall(x, y) \in \bS, \rho(\randKappa'(\bS_I,n))(x) = y\;
        \wedge\; \Risk_\cD(\rho(\randKappa'(\bS_I,n))) \geq \alpha
    \end{gathered}\right]\\
    &\quad\leq \sum_I \Pr_{\substack{\randKappa \sim \cD_\kappa,\\ \bS_I \sim \cD^{n/2}, \\ \bar{\bS}_I \sim \cD^{n/2}}}\left[\begin{gathered}
        \forall(x, y) \in \bar{\bS}_I, \rho(\randKappa(\bS_I,n))(x) = y\;
        \wedge\; \Risk_\cD(\rho(\randKappa(\bS_I,n))) \geq \alpha
    \end{gathered}\right]\\
    &\quad\leq \binom{2s}{s} \exp(-\alpha n/2).
\end{align*}

Overall, we conclude that
\begin{align*}
    \Pr_{\substack{\randKappa \sim \cD_\kappa,\\ \bS \sim \cD^n}}[\Risk_\cD(\rho(\randKappa(\bS,n))) \geq \alpha]
    \leq \delta + \binom{2s}{s} \exp(-\alpha n/2).
\end{align*}
Finally, we obtain the thesis by considering $\beta \geq 2 \delta$ and choosing $\alpha = 2(s \ln(4) + \ln(2/\beta))/n$ so that $\binom{2s}{s} \cdot \exp(-\alpha n/2) \leq \beta/2$.

\end{proof}

    \section{Efficient Boosting via Randomized Compression}
In this section, we present our proof that Algorithm~\ref{alg:adaboost} achieves the sample complexity stated in Theorem~\ref{thm:boostintro}.
Recall that we are given access to a $\gamma$-weak learner $\cW$.
For any data set $S \in (\cX \times \{-1,1\})^*$ and distribution $\cD$ over $S$, we can query the weak learner with $S$ and $\cD$ and it will return a hypothesis $h\colon \cX \to \{-1,1\}$ such that $\Risk_\cD(h) \leq 1/2-\gamma$.
We assume the hypotheses returned by the weak learner belong to a hypothesis set $\cH$ of VC-dimension $d$.

The parameter $N$ in Algorithm~\ref{alg:adaboost} is an upper bound on $\abs{S} = n$. It is merely used for sake of analysis when invoking the stable compression framework.
It ensures that $K$ remains the same if the algorithm is executed on a subset $S'$ of the training set with the same value of $N$.
When using the algorithm, one should simply set $N$ to $n$.

At a high level, the algorithm runs AdaBoost with a few twists.
We maintain weighted distributions $\bD_k$ over the training data.
In each step, the weak learner is invoked to obtain a hypothesis $\bh_k$ with a small error under distribution $\bD_k$.
However, unlike in AdaBoost, we do not invoke the weak learner on the full training data.
Instead, we obtain $\bh_k$ by sampling some $m = O((d+\ln(1/\gamma)) \gamma^{-2})$ data points, denoted $\bS_k$, from $\bD_k$ and train on $\bS_k$ with a uniform weighing.
Furthermore, where AdaBoost would normally update all weights by $e^\alpha$ or $e^{-\alpha}$ for $\alpha = \alpha_k = (1/2)\ln((1-\Risk_{\bD_k}(\bh_k))/\Risk_{\bD_k}(\bh_k))$, we simply fix $\alpha$ as if $\Risk_{\bD_k}(\bh_k)$ was $1/2-\gamma/2$.

\subsection{Corresponding Randomized Compression Scheme}
\label{sec:correspondingscheme}
We now argue that Algorithm~\ref{alg:adaboost} naturally corresponds to a randomized compression scheme.
Let $S = \bigl((x_1,y_1),\allowbreak \ldots, (x_n,y_n)\bigr)$ be the training sequence and $N \geq n$.
Consider an execution of the randomized Algorithm~\ref{alg:adaboost} and let $\bh_1,\dots,\bh_K$ be the hypotheses obtained. From such an execution, we define an encoding map $\randKappa$ that maps $(S,N)$ to the sequence $\bS_{1} \circ \cdots \circ \bS_{K}$, where $\circ$ denotes concatenation and $\bS_i$ is the sample associated with $\bh_i$ (see Line~\ref{line:sample}).
The randomized algorithm thus gives a distribution $\cD_\kappa$ over such encoding maps.

Our reconstruction function $\rho$ on a sequence of $K \cdot m$ samples partitions the samples into $K$ consecutive groups $S_1, \dots, S_K$ of $m$ samples.
It then invokes the weak learner $\cW$ on each $S_i$ with the uniform distribution to obtain $h_i$ and finally produces the function mapping any $x \in \cX$ to $\sign((1/K)\sum_{k=1}^K h_k(x))$.

Notice that $\rho(\randKappa(S,N))(x) = \sign(\brf(x))$, i.e.~the reconstruction function makes the same predictions as the returned voting classifier.
Hence if we can show that the obtained randomized compression scheme has a small failure probability and is stable, then we may use Theorem~\ref{thm:genstable} to bound the generalization error of Algorithm~\ref{alg:adaboost}.
In particular, our compression scheme has size $O(Km)$.
Combining this bound on the size with Theorem~\ref{thm:genstable} proves Theorem~\ref{thm:boostintro}.

In the following, we first argue that the obtained compression scheme has failure probability at most $\delta$ (Lemma~\ref{lem:fmargin}).
We then argue that it is indeed stable (Lemma~\ref{lemma:stability}).

\subsection{Small Failure Probability}
\label{sec:failureprobability}
We show that for any training set $S$, with good probability over the execution of Algorithm~\ref{alg:adaboost} with $N \geq \abs{S} = n$, the returned voting classifier $\brf(x) = (1/K)\sum_{i=1}^K \bh_i(x)$ has large margins on all the training data $S$.
Thus, we can apply Eq.~\ref{eq:marginbased:ub} to it.
Moreover, this also implies that $\sign(\brf)$ has zero empirical error, bounding the failure probability of the algorithm.
Concretely, we show:
\begin{lemma}
\label{lem:fmargin}
    For any training set $S = \left((x_1,y_1), \ldots, (x_n, y_n)\right)$, it holds with probability at least $1-\delta$ over the execution of Algorithm~\ref{alg:adaboost} with $N \geq n$ that the voting classifier $\brf(x)=(1/K)\sum_{i=1}^K \bh_i(x)$ satisfies,  for all $i \in [n]$, that $y_i \brf(x_i) \geq \gamma/128$, and, in particular, that $\sign(\brf(x_i)) = y_i$.
\end{lemma}
The proof of Lemma~\ref{lem:fmargin} makes use of the notion of an $\eps$-approximation. For a concept $c\colon \cX \to \{-1, 1\}$, a hypothesis set $\cH$ and a distribution $\cD$ over $\cX$, a set of samples $S$ is an $\eps$-approximation for $(c, \cD, \cH)$ if for all $h \in \cH$, it holds that 
\[
    \left\lvert\Pr_{x \sim \cD}[h(x) \neq c(x)] - \frac{|\{ x \in S : h(x) \neq c(x)\}|}{|S|}\right\rvert \leq \eps.
\]
The following result ensures that a large enough set of samples $\bS \sim \cD^n$ is an $\eps$-approximation with good probability.
\begin{theorem}[\citealt{lls,talagrand,vapnik71uniform}]
\label{thm:sample}
    There exists universal constant $b > 0$, such that for any $0 < \eps,\delta < 1$, any concept $c\colon \cX \to \{-1,1\}$, any $\cH \subseteq \cX \to \{-1,1\}$ of VC-dimension $d$ and any distribution $\cD$ over $\cX$, it holds with probability at least $1-\delta$ over a set $\bS \sim \cD^n$ that $\bS$ is an $\eps$-approximation for $(c,\cD,\cH)$ provided that $n \geq b((d+\ln(1/\delta))\eps^{-2})$.
\end{theorem}

We now present our formal argument.
\begin{proof}[of Lemma~\ref{lem:fmargin}]
Fix any set $S$ of $n$ samples $(x_1,y_1),\dots,(x_n,y_n)$ and let $c\colon (\cX \cap S) \to \{-1,1\}$ denote the concept with $c(x_i)=y_i$ for each $i=1,\dots,n$. 

Define an indicator random variable $\bX_k$ for each step $k=1,\dots,K$ taking the value $1$ if $\bS_k$ fails to be a $\gamma/2$-approximation for $(c,\bD_k,\cH)$. Note that for any outcome $S_1,\dots,S_{k-1}$ of the random samples $\bS_1,\dots,\bS_{k-1}$, we get from Theorem~\ref{thm:sample} and our choice of $m = a((d + \ln(1/\gamma)) \gamma^{-2})$ that $\Pr[\bX_k=1 \mid \forall i < k : \bS_i=S_i] \leq \gamma^2/32$ for a large enough constant $a>0$. It follows from a Chernoff bound that $\Pr[\sum_i \bX_i > \gamma^2 K/16] \leq \exp(-\gamma^2 K/32) = \delta/(eN) < \delta/2$. Let us now assume that at most $\gamma^2 K/16$ of the samples $\bS_i$ fail to be a $\gamma/2$-approximation. We claim that $\brf(x) = (1/K)\sum_{k=1}^K \bh_k(x)$ satisfies $y_i \brf(x_i) \geq \gamma/128$ in this case.

To see this, consider the exponential loss 
\[
    \sum_{i=1}^n \exp\left(-\alpha y_i\sum_{k=1}^K \bh_k(x_i)\right).
\]
We compare this to the final weights $\bD_{K+1}$. Since $\bD_{K+1}$ is a probability distribution, we have
\begin{align*}
    1
    &= \sum_{i=1}^n \bD_{K+1}(i) \\
    &= \sum_{i=1}^n \frac{\bD_{K}(i) \exp(-\alpha y_i \bh_K(x_i))}{\bZ_k} \\
    &= \frac{1}{n} \sum_{i=1}^n \frac{\exp(-\alpha  y_i \sum_{k=1}^K \bh_k(x_i))}{\prod_{k=1}^K \bZ_k}.
\end{align*} 
From this, we observe that 
\[
    \sum_{i=1}^n \exp\left(-\alpha y_i\sum_{k=1}^K \bh_k(x_i)\right) = n \prod_{k=1}^K \bZ_k.
\]
To bound the $\bZ_k$, we analyze two cases. First, if $\bX_k=0$, then we know that $\bS_k$ is a $\gamma/2$-approximation for $\bD_k$. Furthermore, since $\cW$ is a $\gamma$-weak learner, we have that $\Risk_{\bS_k}(\bh_k) \leq 1/2-\gamma$ where $\Risk_{\bS_k}(\bh_k)$ denotes the fraction of mispredictions among samples in $\bS_k$. By the definition of a $\gamma/2$-approximation, this further implies $\Risk_{\bD_k}(\bh_k) \leq 1/2 - \gamma/2$. If $\bX_k=1$, then we simple bound $\Risk_{\bD_k}(\bh_k) \leq 1$.

We now observe that
\begin{align*}
    \bZ_k
    &= \sum_{i=1}^m \bD_k(i) \exp(-\alpha  y_i \bh_k(x_i)) \\
    &= \sum_{i : \bh_k(x_i)\neq y_i} \bD_k(i)e^{\alpha} + \sum_{i : \bh_k(x_i) = y_i} \bD_k(i) e^{-\alpha} \\
    &= \Risk_{\bD_k}(\bh_k) e^\alpha + (1-\Risk_{\bD_k}(\bh_k))e^{-\alpha}.
\end{align*}
For $\bX_k=0$, this is upper bounded by
\begin{align*}
    \bZ_k
    &\leq (1/2-\gamma/2) e^\alpha + (1/2+\gamma/2)e^{-\alpha} \\
    &= 2 \sqrt{(1/2-\gamma/2)(1/2+\gamma/2)} \\
    &= \sqrt{1 - \gamma^2}.
\end{align*}
For $\bX_k=1$, it is upper bounded by 
\begin{align}
    \bZ_k
    &\leq e^\alpha \nonumber\\
    &= \sqrt{(1/2+\gamma/2)/(1/2-\gamma/2)} \nonumber\\
    &\leq \sqrt{1 + \frac{\gamma}{1/2-\gamma/2}} \nonumber\\
    &\leq \sqrt{1 + 4 \gamma}
    .
    \label{eq:ub:Zk}
\end{align}
Using that $\sum_{k=1}^K \bX_k \leq \gamma^2 K/16$, we thus conclude 
\begin{align*}
    \prod_{k=1}^K \bZ_k
    &\leq (1-\gamma^2)^{(K-\gamma^2 K/16)/2}(1+4\gamma)^{\gamma^2 K/32} \\
    &\leq \exp\left(\gamma^3 K/8 - \gamma^2(K-\gamma^2 K/16)/2 \right) \\
    &\leq \exp(-\gamma^2 K/4) \\
    &\leq (\delta/N)^2.
\end{align*}
We therefore have 
\[
    \sum_{i=1}^n \exp\left(-\alpha y_i\sum_{k=1}^K \bh_k(x_i)\right)
    \leq \delta/N
    ,
\]
so, by non-negativity of the exponential function, $\exp(-\alpha y_i \sum_{k=1}^K \bh_k(x_i)) \leq \delta/N$ for all $i \in [n]$.
Raising both sides of the inequality to the power $1/(K \alpha)$ gives $\exp(-y_i \brf(x_i)) \leq (\delta/N)^{1/K \alpha}$, so $y_i \brf(x_i) \geq \ln(N/\delta)/(K\alpha)$.
From Eq.~\eqref{eq:ub:Zk}, we have that $e^\alpha \leq \sqrt{1 + 4 \gamma}$, hence $\alpha \leq (1/2)\ln(1+4 \gamma) \leq (1/2)\ln(e^{4\gamma}) = 2\gamma$.
Thus, we conclude that $y_i \brf(x_i) \geq \ln(N/\delta)/(K 2 \gamma) \geq \gamma/128$.
\end{proof}

\subsection{Stability}
\label{sec:stability}
In the following, we show the stability of the compression scheme corresponding to Algorithm~\ref{alg:adaboost}.

Fix a $\gamma$-weak learner $\cW$, a failure probability $\delta$, and an upper bound $N$ on the size of the training set.
Given $S \in \cup_{i = 1}^N (\cX \times \cY)^i$,
let $\methodname{Exec}(S,N) = \bS_1, \ldots, \bS_K$ denote the sequence of samples associated with the execution of Algorithm~\ref{alg:adaboost} on input $S, \cW, \delta, N$.
In this way, the sequence $\bS_i$ is the sample drawn at Line~\ref{line:sample} on the $i$th iteration of the \textbf{for} loop starting at Line~\ref{line:forloop}.
The randomized compression scheme $\randKappa$ underlying Algorithm~\ref{alg:adaboost}, as discussed in Section~\ref{sec:correspondingscheme}, can then be described by $\randKappa(S,N) = \bS_1 \circ \cdots \circ \bS_K$.

\begin{lemma}
    \label{lemma:stability}
    The randomized compression scheme $\randKappa$ given by $\randKappa(S,N) = \methodname{Exec}(S,N)$ is stable.
\end{lemma}
\begin{proof}
    Given $n \in [N]$, let $S \in (\cX \times \cY)^n$, and let $S'$ be a subsequence of $S$.
    Let $\methodname{Exec}(S,N) = \bS_1, \ldots, \bS_K$ and $\methodname{Exec}(S',N) = \bS'_1, \ldots, \bS'_K$.
    We will show that
    for all $k \in [K]$
    it holds that conditioning on $\bS_i \sqsubseteq S'$ for $i \in [k]$ implies that $\bS_1 \circ \cdots \circ \bS_k$ follows the same distribution as $\bS'_1 \circ \cdots \circ \bS'_k$.
    We argue by induction on $k$ and conclude the thesis by considering $k = K$.

    For the base case,
    we have that $\bS_1$ consists of $m$ i.i.d.~samples from the uniform distribution over $S$.
    Therefore, conditioning on $\bS_1 \sqsubseteq S'$ makes the $m$ samples i.i.d.~following the uniform distribution over $S'$ and, thus, makes $\bS_1$ identically distributed to $\bS'_1$ (this uses our assumption that $S$ contains no duplicates).

    Now, for the induction step, suppose that for some $k \in [K-1]$ we have that, for all $T \sqsubseteq S$,
    \begin{align*}
        \Prob[\big]{\bS_1 \circ \cdots \circ \bS_k = T \given \bS_i \sqsubseteq S' \;\forall i \in [k]}
        &= \Prob[\big]{\bS'_1 \circ \cdots \circ \bS'_k = T}.
    \end{align*}
    We consider $T \sqsubseteq S'$ since otherwise both sides of the equation are zero.
    For $i \in [k+1]$, let $\bD_i$ and $\bh_i$ be the distribution (see Line~6) and hypothesis (see Line~7) corresponding to the $i$th iteration of the \textbf{for} loop starting at Line~5 when executing Algorithm~\ref{alg:adaboost} on input $S, \cW,\delta, N$.
    Define $\bD_i'$s and $\bh_i'$s associated with the execution on $S', \cW, \delta, N$ analogously.

    For the remainder of the proof, we condition on the event that $\bS_i \sqsubseteq S'$ for all $i \in [k]$.
    The induction hypothesis implies that $\bS_1,\dots,\bS_{k}$ and $\bS'_1,\dots,\bS'_{k}$ follow the same distribution. Now fix any $T_{k}=S_1,\dots,S_{k}$ in the support of this distribution. Note that conditioning on $T_{k}$ fixes the hypotheses $\bh_1,\dots,\bh_{k}$ and $\bh'_1,\dots,\bh'_k$ to the same fixed $h_1,\dots,h_{k}$. This further fixes $\bD_{k+1}$ to $D_{k+1}(j) = \exp(-\alpha y_j \sum_{\ell=1}^k h_\ell(x_j)) / Z$ where $Z$ is a normalization factor making $D_{k+1}$ a probability distribution. Similarly for $S'$, it fixes $\bD'_{k+1}$ to $D_{k+1}'(j)=\exp(-\alpha y_j \sum_{\ell=1}^k h_\ell(x_j)) / Z'$ for the $j \in S'$.
    
    The crucial observation is that any $x_j$ occurring in both $S'$ and $S$ have the same weight in $D_{k+1}$ and $D'_{k+1}$ up to the normalization factors $Z$ and $Z'$. This implies that if we further condition on $\bS_{k+1} \sqsubseteq S'$, the samples in $\bS_{k+1}$ are i.i.d.\ from $D_{k+1}$ but where every $j \notin S'$ has $D_{k+1}(j)=0$ and the resulting distribution is scaled accordingly. This makes the distribution identical to $D'_{k+1}$ (using the assumption that $S$ contains no duplicates), which concludes the proof.
\end{proof}

    \section{Conclusion}
In this work, we took a first step towards developing voting classifiers with an optimal sample complexity for weak-to-strong learning.
Concretely, we improve the dependency on the number of samples $n$ by a logarithmic factor over previous works.
To analyze our new algorithm, we further introduce a new framework of randomized compression schemes that we hope may prove useful in future work.

Our work leaves open a number of intriguing directions to pursue.
First, can we develop a voting classifier with an optimal sample complexity as in Eq.~\eqref{eq:opt}?
Or, as a first and more modest goal, can we develop a voting classifier with only a single logarithmic sub-optimal dependency on $n$, like our Algorithm~\ref{alg:adaboost}, but with an optimal dependency on the remaining parameters $d$, $\gamma$, and $\delta$?
Another question is whether our analysis of Algorithm~\ref{alg:adaboost} is tight, or could it perhaps be improved to yield an even better sample complexity?
Also, for previous algorithms such as AdaBoost, the current best analysis gives a sample complexity as in Eq.~\eqref{eq:ada} with two logarithmic factors of sub-optimality.
Can the analysis be improved for some of those algorithms?
We know that it can never be improved to an optimal sample complexity (in light of~\cite{adaboostNotOptimal}, see the discussion in Section~\ref{sec:other}), but perhaps one of the logarithmic factors can be removed.
The same holds for the uniform convergence bounds for large-margin voting classifiers.
Can these be improved by a logarithmic factor?

    \acks{%
This research is co-funded by the European Union (ERC, TUCLA, 101125203) and Independent Research Fund Denmark (DFF) Sapere Aude Research Leader Grant No.~9064-00068B.
Views and opinions expressed are however those of the author(s) only and do not necessarily reflect those of the European Union or the European Research Council.
Neither the European Union nor the granting authority can be held responsible for them.

Parts of this research was done while Martin Ritzert was supported by DIREC – Digital Research Centre Denmark.
}

    \bibliography{res}

    \appendix

\end{document}